\documentclass[11pt,a4paper]{article}
\usepackage[hyperref]{acl2020}
\usepackage{xr}

\usepackage{graphicx}
\usepackage{xcolor}
\definecolor{color1}{rgb}{0.4, 0., 0.} 
\definecolor{color2}{rgb}{0.6, 0.4, 1} 

\usepackage{times}
\usepackage{latexsym}
\usepackage[utf8]{inputenc}
\usepackage{xcolor}
\usepackage{linguex}
\usepackage{booktabs}
\usepackage{multirow} 
\usepackage{multicol}
\usepackage{url}
\usepackage{amsmath}
\usepackage{amssymb}
\usepackage{algorithm}
\usepackage[algo2e, ruled]{algorithm2e} 
\usepackage{tabularx}
\usepackage{amsthm}

\usepackage{esvect}
\usepackage{array}
\usepackage[export]{adjustbox}
\usepackage{subfig}
\usepackage{adjustbox}
\usepackage{fancyhdr}

\usepackage[margin=1in,headsep=.2in]{geometry}

\newcommand{\R}{\mathbb{R}}

\makeatletter
\newcommand*{\addFileDependency}[1]{
  \typeout{(#1)}
  \@addtofilelist{#1}
  \IfFileExists{#1}{}{\typeout{No file #1.}}
}
\makeatother

\usepackage{microtype}

\newtheorem{theorem}{Theorem}[section]
\newtheorem{corollary}{Corollary}[theorem]
\newtheorem{lemma}[theorem]{Lemma}

\usepackage{adjustbox} 

\def\aclpaperid{1562X} 

\aclfinalcopy

\title{Null It Out: Guarding Protected Attributes by Iterative Nullspace Projection}

\author{Shauli Ravfogel\textsuperscript{1,2} \;\;\; Yanai Elazar\textsuperscript{1,2} \;\;\; Hila Gonen\textsuperscript{1}  \;\;\; Michael Twiton\textsuperscript{3}  \;\;\; Yoav Goldberg\textsuperscript{1,2}\\
\textsuperscript{1}Computer Science Department, Bar Ilan University \\
\textsuperscript{2}Allen Institute for Artificial Intelligence \\
\textsuperscript{3}Independent researcher \\
  {\tt  \{shauli.ravfogel, yanaiela, hilagnn, mtwito101, yoav.goldberg\}@gmail.com}
  }

\date{}

\begin{document}


\maketitle


\begin{abstract}

The ability to control for the kinds of information encoded in neural representation has a variety of use cases, especially in light of the challenge of interpreting these models. 
We present Iterative Null-space Projection (INLP), a novel method for removing information from neural representations. Our method is based on repeated training of linear classifiers that predict a certain property we aim to remove, followed by projection of the representations on their null-space. By doing so, the classifiers become oblivious to that target property, making it hard to linearly separate the data according to it.
While applicable for multiple uses, we evaluate our method on bias and fairness use-cases, and show that our method is able to mitigate bias in word embeddings, as well as to increase fairness in a setting of multi-class classification.


\end{abstract}

\section{Introduction}

What is encoded in vector representations of textual data, and can we control
it?
Word embeddings, pre-trained language models, and more generally deep learning
methods emerge as very
effective techniques for text classification. Accordingly, they are increasingly
being used for predictions in real-world situations.
A large part of the success is due to the
models' ability to perform \emph{representation learning}, coming up with
effective feature representations for the prediction task at hand. However,
these learned representations, while effective, are also notoriously opaque: we do not
know what is encoded in them.  Indeed, there is an emerging line of work on
probing deep-learning derived representations for syntactic \cite{linzen2016agreement,structural-probe,goldberg2019assessing}, semantic \cite{tenney2019bert} and factual knowledge \cite{petroni2019language}. There is also evidence that they capture a lot of
information regarding the demographics of the author of the text \cite{blodgett2016demographic,elazar2018}.

\begin{figure}[t]
    \centering
    \includegraphics[width=1.\columnwidth]{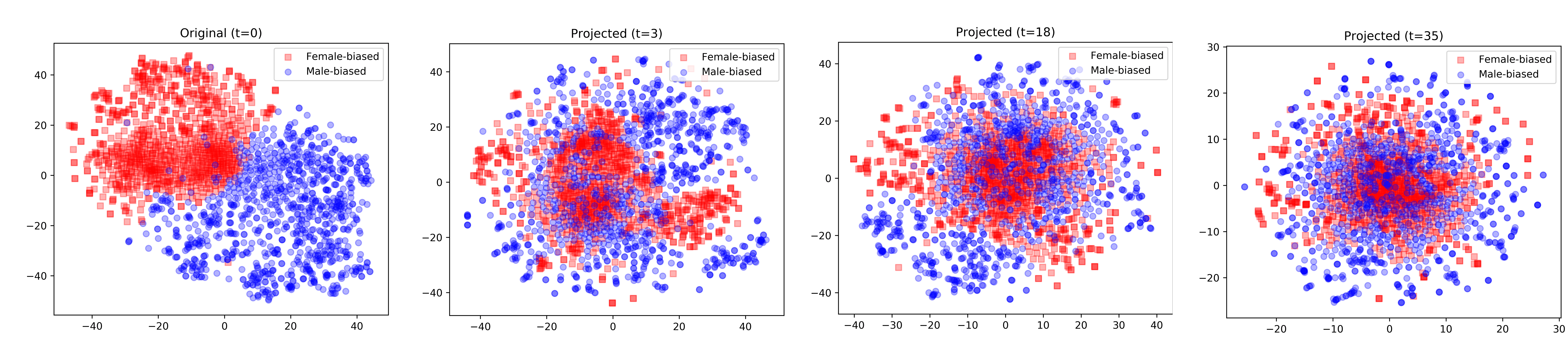}
    \caption{t-SNE projection of GloVe vectors of the most gender-biased words after t=0, 3, 18, and 35 iterations of INLP. Words are colored according to being male-biased or female-biased.}
    \label{fig:tsne}
\end{figure}

What can we do in situations where we \emph{do not want} our representations to encode
certain kinds of information? For example, we may want a word representation
that does not take \emph{tense} into account, or that does not encode
\emph{part-of-speech} distinctions. We may want a classifier that judges the
\emph{formality} of the text, but which is also oblivious to the \emph{topic}
the text was taken from. Finally, and also our empirical focus in this work,
this situation often arises when considering \emph{fairness} and \emph{bias} of
language-based classification. We may not want our word-embeddings to encode \emph{gender}
stereotypes, and we do not want sensitive decisions on hiring or loan approvals
to condition on the \emph{race}, \emph{gender} or \emph{age} of the applicant.

We present a novel method for selectively removing specific kinds of information
from a representation. Previous methods are either
based on projection on a pre-specified, user-provided direction \cite{bolukbasi2016man},
or on adding an adversarial objective to an end-to-end training process \cite{xie2017controllable}.
Both of these have benefits and limitations, as we discuss in the related work section
(\S{\ref{sec:related-work}}). Our proposed method, Iterative Null-space Projection
(INLP), presented in section \ref{sec:method}, can be seen as a combination of these
approaches, capitalizing on the benefits of both. 
Like the projection methods, it is also based on the mathematical notion
of linear projection, a commonly used deterministic operator. Like the adversarial methods,
it is data-driven in the directions it removes: we do not presuppose specific directions in the latent space that correspond to the protected attribute, but rather \emph{learn} those directions, and remove them. Empirically, we find it to work well.
We evaluate the method on the challenging task of removing gender signals from
word embeddings \cite{bolukbasi2016man,gn-glove}. Recently, \citet{gonen2019lipstick} showed several limitations
of current methods for this task.  We show that our method is effective in reducing many,
but not all, of these (\S\ref{sec:debias-word-vectors}).

We also consider the context of fair classification, where we want to ensure
that a classifier's decision is oblivious to a protected attribute such as race,
gender or age. There, we need to integrate the projection-based method within a pre-trained classifier. We
propose a method to do so in section \S\ref{sec:deep-debiasing}, and demonstrate
its effectiveness in a controlled setup (\S\ref{sec:deepmoji}) as well as in a real-world one (\S\ref{sec:bios}). 

Finally, while we propose a general purpose information-removal method, our main
evaluation is in the realm of bias and fairness applications. We stress that
this calls for some stricter scrutiny, as the effects of blindly trusting strong
claims can have severe real-world consequences on individuals. We discuss
the limitations of our model in the context of such applications in section
\S\ref{sec:limitations}.

\section{Related Work}
\label{sec:related-work}


The objective of controlled removal of specific types of information from neural representation is tightly related to the task of disentanglement of the representations \citep{bengio-representation-learning, lecum-disentangling}, that is, controlling and separating the different kinds of information encoded in them. In the context of transfer learning, previous methods have pursued representations which are \emph{invariant} to some properties of the input, such as genre or topic, in order to ease domain transfer \citep{ganin-adverserial-adaptation}. Those methods mostly rely on adding an adversarial component \citep{goodfellow-gan, ganin-adverserial-adaptation, xie2017controllable,zhang2018mitigating} to the main task objective: the representation is regularized by an adversary network, that competes against the encoder, trying to extract the protected information from its representation. 

While adverserial methods showed impressive performance in various machine learning tasks, and were applied for the goal of removal of sensitive information \cite{elazar2018,coavoux2018privacy,reshef19,barrett2019adversarial}, they are notoriously hard to train. \citet{elazar2018} have evaluated adverserial methods for the removal of  demographic information from representations. They showed that the complete removal of the protected information is nontrivial: even when the attribute seems protected, different classifiers of the same architecture can often still succeed in extracting it. Another drawback of these methods is their reliance on a main-task loss in addition to the adverserial loss, making them less suitable for tasks such as debiasing pre-trained word embeddings.


\citet{nullspace-cleaning} utilized a "nullspace cleaning" operator for increasing privacy in classifiers. They remove from the input a subspace that contains (but is not limited to) the nullspace of a pre-trained classifier, in order to clean information that is \emph{not} used for the main task (and might be protected), while minimally impairing classification accuracy. While similar in spirit to our method, several key differences exist. As the complementary setting -- removing the nullsapce of the main-task classifier vs. projection onto the nullspace of protected attribute classifiers -- aims to achieve a distinct goal (privacy preserving), there is no notion of exhaustive cleaning. Furthermore, they do not remove protected attributes that \emph{are} used by the classifier (e.g. when it conditions on gender).

A recent line of work focused on projecting the representation to a subspace which does not encode the protected attributes. Under this method, one identifies directions in the latent space that correspond to the protected attribute, and removes them. In a seminal work, \citet{bolukbasi2016man} aimed to identify a ``gender subspace" in word-embedding space by calculating the main directions in a subspace spanned by the differences between gendered word pairs, such as $\vv{he} - \vv{she}$. They suggested to zero out the components of neutral words in the direction of the ``gender subspace" first principle components, and actively pushed neutral words to be equally distant from male and female-gendered words. However, \citet{gonen2019lipstick} have shown that these methods only cover up the bias and that in fact, the information is deeply ingrained in the representations. A key drawback of this approach is that it relies on an intuitive selection of \emph{a few} (or a single) gender directions, while, as we reveal in our experiments, the gender subspace is actually spanned by dozens to hundreds of orthogonal directions in the latent space, which are not necessarily as interpretable as the $\vv{he} - \vv{she}$ direction. This observation aligns with the analysis of \citet{ethayarajh2019understanding} who demonstrated that debiasing by projection is theoretically effective provided that one removes \emph{all} relevant directions in the latent space.   

\begin{figure}[t]
    \centering
    \includegraphics[width=0.8\columnwidth]{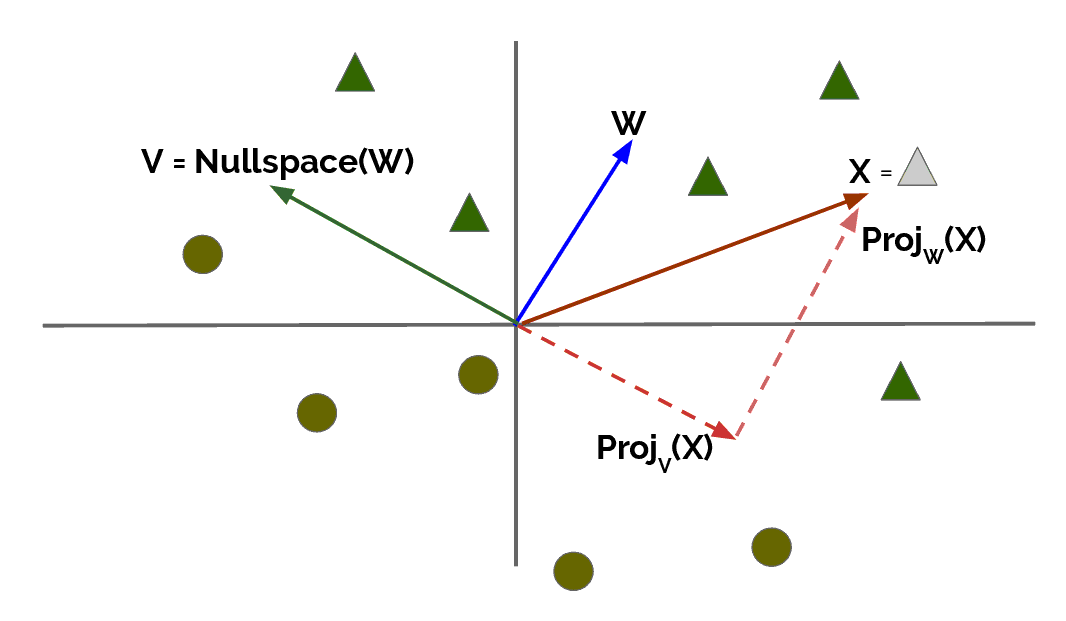}

    \caption{Nullspace projection for a 2-dimensional binary classifier. The decision boundary of $W$ is $W$ 's null-space.}
    \label{fig:nullspace-projection}
\end{figure}


\section{Objective and Definitions}

Our main goal is to ``guard'' sensitive information, so that it will not be encoded in a representation.
Given a set of vectors $x_i \in \R^d$, and corresponding discrete attributes $Z$, $z_i \in \{1, ..., k\}$ (e.g. race or gender), we aim to learn a transformation $g : \R^d \to \R^d$, such that $z_i$ cannot be predicted from $g(x_i)$. In this work we are concerned with ``linear guarding'': we seek a guard $g$ such that no linear classifier $w(\cdot)$ can predict $z_i$ from $g(x_i)$ with an accuracy greater than that of a decision rule that considers only the proportion of labels in $Z$. We also wish for $g(x_i)$ to stay informative:
when the vectors $x$ are used for some end task, we want $g(x)$ to have as minimal influence as possible on the end task performance, provided that $z$ remains guarded.
We use the following definitions:
\paragraph{Guarded w.r.t. a hypothesis class}
Let $X=x_1,...,x_m \in \mathcal{X} \subseteq \R^d$ be a set of vectors, with
corresponding discrete attributes $Z$, $z_i \in \{1,...,k\}$. We say the set
$X$ is \emph{guarded for Z with respect to hypothesis class $\mathcal{H}$} (conversely \emph{Z is guarded in X}) if there is no
classifier $W\in \mathcal{H}$ that can predict $z_i$ from $x_i$ at better than
guessing the majority class. 

\paragraph{Guarding function}
A function $g: \R^n \to \R^n$ is said to be \emph{guarding X for Z (w.r.t. to class $\mathcal{H}$)} if the set
$\{g(x) | x \in X\}$ is guarded for $Z$ w.r.t. to $\mathcal{H}$.

We use the term \textbf{linearly guarded} to indicate guarding w.r.t. to the class of all linear classifiers.




\section{Iterative Nullspace Projection}
\label{sec:method}

Given a set of vectors $x_i \in \R^d$ 
and a set of corresponding discrete\footnote{While this work
focuses on the discrete case, the extension to a linear regression setting is straightforward: A projection to the nullspace of a linear regressor $w$ enforces $wx=0$ for every $x$, i.e., each input is regressed to the non-informative value of zero.} protected attributes $z_i \in \mathcal{Z}$, 
we seek a \textit{linear guarding function} $g$ that remove the linear dependence between $\mathcal{Z}$ and $\mathcal{X}$. 

  
We begin with a high-level description of our approach. Let $c$ be a trained linear classifier, parameterized by a matrix $W \in \R^{k \times d}$, that predicts a property $z$ with some accuracy. We can construct a projection matrix $P$ such that $W(Px)=0$ for all $x$, rendering $W$ useless on dataset $\mathcal{X}$. We then iteratively train additional classifiers $W'$ and perform the same procedure, until no more linear information regarding $\mathcal{Z}$ remains in $X$. Constructing $P$ is achieved via nullspace projection, as described below.
This method is the core of the INLP algorithm (Algorithm \ref{algo:projection}).   

\paragraph{Nullspace Projection} 
The linear interaction between $W$ and a new test point $x$ has a simple geometric interpretation: $x$ is projected on the subspace spanned by $W$'s rows, and is classified according to the dot product between $x$ and $W$'s rows, which is proportional to the components of $x$ in the direction of $W$'s rowpsace. Therefore, if we zeroed all components of $x$ in the direction of $W$'s row-space, we removed all information used by $W$ for prediction: the decision boundary found by the classifier is no longer useful. As the orthogonal component of the rowspace is the nullspace, zeroing those components of $x$ is equivalent to projecting $x$ on $W$'s nullspace. 
Figure \ref{fig:nullspace-projection} illustrates the idea for the 2 dimensional binary-classification setting, in which $W$ is just a 2-dimensional vector.

For an algebraic interpretation, recall that the null-space of a matrix $W$ is defined as the space $N(W) = \{ x | Wx = 0\}$. Given the basis vectors of $N(W)$ we can construct a projection matrix $P_{N(W)}$ into $N(W)$, yielding $W(P_{N(W)}x) = 0 \,\,\, \forall x$. 

This suggests a simple method for rendering $z$ linearly guarded for a set of vectors $X$: training a linear classifier that is parameterized by $W_0$ to predict $Z$ from $X$, calculating its nullspace, finding the orthogonal projection matrix $P_{N(W_0)}$ onto the nullspace, and using it to remove from $X$ those components that were used by the classifier for predicting $Z$. 

Note that the orthogonal projection $P_{N(w_0)}$ is the \emph{least harming linear operation} to remove the linear information captured by $W_0$ from $X$, in the sense that among all maximum rank (which is not full, as such transformations are invertible---hence not linearly guarding) projections onto the nullspace of $W_0$, it carries the least impact on distances. This is so since the image under an orthogonal projection into a subspace is by definition the closest vector in that subspace. 


\paragraph{\emph{Iterative} Projection}
\label{sec:debias-word-vectors}

Projecting the inputs $X$ on the nullspace of a single linear classifier does not suffice for making $Z$ linearly guarded: classifiers can often still be trained to recover $z$ from the projected $x$ with above chance accuracy, as there are often multiple linear directions (hyperplanes) that can partially capture a relation in multidimensional space.
This can be remedied with an iterative process: 
After obtaining $P_{N(W_0)}$, we train classifier $W_1$ on $P_{N(W_0)}X$, obtain a projection matrix $P_{N(W_1)}$, train a classifier $W_2$ on $P_{N(W_1)}P_{N(W_0)}X$ and so on, until no classifier $W_{m+1}$ can be trained. We return the guarding projection matrix $P=P_{N(W_m)}P_{N(W_{m-1})}...P_{N(W_0)}$, with the guarding function $g(x) = Px$. Crucially, the $i$th classifier $W_i$ is trained on the data $X$ after the projection on the nullspaces of classifiers $W_0$, ..., $W_{i-1}$ and is therefore trained to find separating planes that are independent of the separating planes found by previous classifiers.

In Appendix \S\ref{INLP-proofs} we prove three desired proprieties of INLP: (1) any two protected-attribute classifiers found in INLP are orthogonal (Lemma \ref{lem:orthogonality}); (2) while in general the product of projection matrices is not a projection, the product $P$ calculated in INLP is a valid projection 
(Corollary \ref{lem:is-projection}); and (3) it projects any vector to the intersection of the nullspaces of each of the classifiers found in INLP, that is, after $n$ INLP iterations, $P$ is a projection to $N(W_0) \cap N(W_1) \dots \cap N(W_n)$ (Corollary \ref{lem:is-projection-to-intersection}). We further bound the damage $P$ causes to the structure of the space (Lemma \ref{lem:distances}). INLP can thus be seen as a linear dimensionality-reduction method, which keeps only those directions in the latent space which are \emph{not} indicative of the protected attribute.

During iterative nullspace projection, the property $z$ becomes increasingly linearly-guarded in $Px$. 
For binary protected attributes, each intermediate $W_j$ is a vector, and the nullspace rank is $d-1$. Therefore, after $n$ iterations, if the original rank of $X$ was $r$, the rank of the projected input $g(X)$ is at least $r-n$.
The entire process is formalized in Algorithm \ref{algo:projection}.

\begin{algorithm}[H]
\SetAlgoLined
\BlankLine
 \SetKwInOut{Input}{Input}  
\Input{$(X,Z)$: a training set of vectors and protected attributes\\
      n: Number of rounds   
}

 \KwResult{A projection matrix $P$}
 
\SetKwFunction{FMain}{GetProjectionMatrix}
\SetKwProg{Fn}{Function}{:}{}
\Fn{\FMain{$X,Z$}}{
 $X_{projected} \gets X$ \\
 $P \gets I$\\
\For{i $\leftarrow 1$ \KwTo n}{
 $  W_i \leftarrow$ TrainClassifier($X_{projected},Z$) \\
 $B_i \leftarrow$ GetNullSpaceBasis($W_i$) \\
 $P_{N(W_i)} \leftarrow$ $B_{i}B{i}^T$ \\
 $P$  $\leftarrow$ $P_{N(W_i)}P$ \\
 $X_{projected}$ $\leftarrow$ $P_{N(W_i)}X_{projected}$ \\
}
\Return P
}
\caption{Iterative Nullspace Projection (INLP)}
\label{algo:projection}
\end{algorithm}

INLP bears similarities to Partial Least Squares (PLS; \citet{geladi1986partial, pls-classification}), an established regression method. Both iteratively find directions that correspond to $Z$: while PLS does so by maximizing covariance with $Z$ (and is thus less suited to classification), INLP learns the directions by training classifiers with an arbitrary loss function. Another difference is that INLP focuses on learning a projection that neutralizes $Z$, while PLS aims to learn low-dimensional representation of $X$ that keeps information on $Z$.

\paragraph{Implementation Details}
\label{implementaton}
A naive implementation of Algorithm \ref{algo:projection} is prone to numerical errors, due to the accumulative projection-matrices multiplication $P \gets P_{N(W_i)}P$. To mitigate that, we use the formula of \citet{benisrael2015projectors}, which connects the intersection of nullspaces with the projection matrices to the corresponding rowspaces:
\vspace{-1em}
\begin{equation}
\resizebox{0.91\hsize}{!}{%
    $N(w_1)  \cap  \dots  \cap  N(w_n) = N(P_R(w_1) + \dots + P_R(w_n))$
    }
    \label{ben-israel-formula}
\end{equation}
Where $P_R(W_i)$ is the orthogonal projection matrix to the row-space of a classifier $W_i$.
Accordingly, in practice, we do not multiply $P \gets P_{N(W_i)}P$ but rather collect \emph{rowspace} projection matrices $P_{R(W_i)}$ for each classifier $W_i$. In place of each input projection $X_{projected} \gets P_{N(W_i)}X_{projected}$, we recalculate $P:=P_{N(w_1) \cap ... \cap N(w_i)}$ according to \ref{ben-israel-formula}, and perform a projection $X_{projected} \gets PX$. Upon termination, we once again apply \ref{ben-israel-formula} to return the final nullspace projection matrix $P_{N(W_1) \cap ... \cap N(W_n)}$. The code is publicly available.\footnote{\href{https://github.com/Shaul1321/nullspace\_projection}{https://github.com/Shaul1321/nullspace\_projection} }

\section{Application to Fair Classification}
\label{sec:deep-debiasing}
The previous section described the INLP method for producing a linearly guarding function $g$ for a set of vectors. We now turn to describe its usage in the context of providing fair classification by a (possibly deep) neural network classifier.

In this setup, we are given, in addition to $X$ and $Z$ also labels $Y$, and wish to construct a classifier $f: X \to Y$, while being \emph{fair} with respect to $Z$. Fairness in classification can be defined in many ways \cite{equal-opportunities+odds,demographic-parity,zhang2018mitigating}. We focus on a notion of fairness by which the predictor $f$ is oblivious to $Z$ when making predictions about $Y$.

To use linear guardedness in the context of a deep network, recall that a classification network $f(x)$ can be decomposed into an encoder $enc$ followed by a linear layer $W$: $f(x) = W \cdot enc(x)$, where $W$ is the last layer of the network and $enc$ is the rest of the network. If we can make sure that $Z$ is linearly guarded in the inputs to $W$, then $W$ will have no knowledge of $Z$ when making its prediction about $Y$, making the decision process oblivious to $Z$. Adversarial training methods attempt to achieve such obliviousness by adding an adversarial objective to make $enc(x)$ itself guarding. We take a different approach and add a guarding function \emph{on top of} an already trained $enc$. 
\\[0.2em]\noindent\textbf{We propose the following procedure.} Given a training set $X$,$Y$ and protected attribute $Z$, we first train a neural network $f=W\cdot enc(X)$ to best predict $Y$. This results in an encoder that extracts effective features from $X$ for predicting $Y$.
We then consider the vectors $enc(X)$, and use the INLP method to produce a linear guarding function $g$ that guards $Z$ in $enc(X)$.
At this point, we can use the classifier $W\cdot g(enc(x))$ to produce oblivious decisions, however by introducing $g$ (which is lower rank than $enc(x)$) we may have harmed $W$s performance. We therefore freeze the network and fine-tune only $W$ to predict $Y$ from $g(enc(x))$, producing the final fair classifier $f'(x) = W'\cdot g(enc(x))$. Notice that $W'$ only sees vectors which are linearly guarded for $Z$ during its training, and therefore cannot take $Z$ into account when making its predictions, ensuring fair classification. 

We note that our notion of fairness by obliviousness does not, in the general case, correspond to other fairness metrics, such as equality of odds or of opportunity. It does, however, \emph{correlate} with fairness metrics, as we demonstrate empirically.
\\[0.2em]\noindent\textbf{Further refinement.}
Guardedness is a property that holds in expectation over an entire dataset. For example, when considering a dataset of individuals from certain professions (as we do in \S\ref{sec:bios}), it is possible that the entire dataset is guarded for gender, yet if we consider only a subset of individuals (say, only nurses), we may still be able to recover gender with above majority accuracy, in that sub-population. As fairness metrics are often concerned with classification behavior also within groups, we propose the following refinement to the algorithm, which we use in the experiments in \S\ref{sec:deepmoji} and \S\ref{sec:bios}:
in each iteration, we train a classifier to predict the protected attribute not on the entire training set, but only on the training examples belonging to a single (randomly chosen) main-task class (e.g. profession). By doing so, we push the protected attribute to be linearly guarded in the examples belonging to each of the main-task labels.

\section{Experiments and Analysis}

\subsection{``Debiasing'' Word Embeddings}
\label{sec:debiasing-word-embeddings}
In the first set of experiments, we evaluate the INLP method in its ability to debias word embeddings \cite{bolukbasi2016man}.
After ``debiasing'' the embeddings, we repeat the set of diagnostic experiments of \citet{gonen2019lipstick}.

\paragraph{Data.} Our debiasing targets are the uncased version of GloVe word embeddings \citep{gn-glove}, after limiting the vocabulary to the 150,000 most common words.
To obtain labeled data $X$,$Z$ for this classifier, we use 
the 7,500 most male-biased and 7,500 most female-biased words (as measured by the projection on the $\vv{he} - \vv{she}$ direction), as well as 7,500 neutral vectors, with a small component (smaller than 0.03) in the gender direction. The data is randomly divided into a test set (30\%), and training and development sets (70\%, further divided into 70\% training and 30\% development examples).

\paragraph{Procedure} We use a $L_2$-regularized SVM classifier \citep{svm} trained to discriminate between the 3 classes: male-biased, female-biased and neutral. We run Algorithm \ref{algo:projection} for 35 iterations.

\subsubsection{Results}

\noindent\textbf{Classification.}
Initially, a linear SVM classifier perfectly discriminates between the two genders (100\% accuracy). The accuracy drops to 49.3\% following INLP. To measure to what extent gender is still encoded in a \emph{nonlinear} way, we train a 1-layer ReLU-activation MLP. The MLP recovers gender with accuracy of 85.0\%. This is expected, as the INLP method is only meant to achieve \emph{linear guarding}\footnote{Interestingly, RBF-kernel SVM (used by \citet{gonen2019lipstick}) achieves random accuracy.}.

\noindent\textbf{Human-selected vs. Learned Directions.} Our method differs from the common projection-based approach by two main factors: the numbers of directions we remove, and the fact that those directions are learned iteratively from data. Perhaps the benefit is purely due to removing more directions? 
We compare the ability to linearly classify words by gender bias after removing 10 directions by our method (running Algorithm \ref{algo:projection} for 10 iterations) with the ability to do so after removing 10 manually-chosen directions defined by the difference vectors between gendered pairs \footnote{We use the following pairs, taken from \citet{bolukbasi2016man}:  (``woman", ``man"), (``girl", ``boy"), (``she", ``he"), (``mother", ``father"), (``daughter", ``son"), (``gal", ``guy"), (``female", ``male"), (``her", ``his"), (``herself", ``himself"), (``mary", ``john").}. INLP-based debiasing results in a very substantial drop in classification accuracy (54.4\%), while the removal of the predefined directions only moderately decreases accuracy (80.7\%). This shows that data-driven identification of gender-directions outperforms manually selected directions: there are many subtle ways in which gender is encoded, which are hard for people to imagine.
\paragraph{Discussion.}
Both the previous method and our method start with the main gender-direction of $\vv{he}-\vv{she}$. However, while previous attempts take this \emph{direction} as the information that needs to be neutralized, our method instead considers the \emph{labeling} induced by this gender direction, and then iteratively finds and neutralizes directions that correlate with this labeling. It is likely that the $\vv{he}-\vv{she}$ direction is one of the first to be removed, but we then go on and \emph{learn} a set of other directions that correlate with the same labeling and which are predictive of it to some degree, neutralizing each of them in turn. Compared to the 10 manually identified gender-directions from \citet{bolukbasi2016man}, it is likely that our learned directions capture a much more diverse and subtle set of gender clues in the embedding space.

\noindent\textbf{Effect of debiasing on the embedding space.} In appendix \S \ref{sec:appendix-neighbors-examples} we provide a list of 40 random words and their closest neighbors, before and after INLP, showing that INLP doesn't significantly damage the representation space that encodes lexical semantics. We also include a short analysis of the influence on a specific subset of inherently gendered words: gendered surnames (Appendix \S \ref{sec:appendix-neighbors-examples-names}).

Additionally, we perform a semantic evaluation of the debiased embeddings using multiple word similarity datasets (e.g. SimLex-999 \cite{hill2015simlex}). We find large improvements in the quality of the embeddings after the projection (e.g. on SimLex-999 the correlation with human judgements improves from 0.373 to 0.489) and we elaborate more on these findings in Appendix \S\ref{sec:appendix-quantitative_emb}.

\noindent\textbf{Clustering.}
Figure \ref{fig:tsne} shows t-SNE \citep{t-SNE} projections of the 2,000 most female-biased and 2,000 most male-biased words, originally and after $t=3$, $t=18$ and $t=35$ projection steps. The results clearly demonstrate that the classes are no longer linearly separable:  this behavior is qualitatively different from previous word vector debiasing methods, which were shown to maintain much of the proximity between female and male-biased vectors \citep{gonen2019lipstick}. To quantify the difference, we perform K-means clustering to $K=2$ clusters on the vectors, and calculate the V-measure \citep{v-measure} which assesses the degree of overlap between the two clusters and the gender groups. For the t-SNE projected vectors, the measure drops from 83.88\% overlap originally, to 0.44\% following the projection; and for the original space, the measure drops from 100\% to 0.31\%. 
\\[0.2em]\noindent\textbf{WEAT.}
While our method does not guarantee attenuating the bias-by-neighbors phenomena that is discussed in \citet{gonen2019lipstick}, it is still valuable to quantify to what extent it does mitigate this phenomenon. We repeat the Word Embedding Association Test (WEAT) from \citet{caliskan2017semantics} which aims to measure the association in vector space between male and female concepts and stereotypically male or female professions. Following \citet{gonen2019lipstick}, we represent the male and female groups with common names of males and females, rather than with explicitly gendered words (e.g. pronouns). Three tests evaluate the association between a group of male names and a groups of female names to (1) career and family-related words; (2) art and mathematics words; and (3) artistic and scientific fields. In all three tests, we find that the strong association between the groups no longer exists after the projection (non-significant p-values of 0.855, 0.302 and 0.761, respectively). 
\\[0.2em]\noindent\textbf{Bias-by-Neighbors.}
To measure bias-by-neighbors as discussed in \citep{gonen2019lipstick}, we consider the list of professions provided in \citep{bolukbasi2016man} and measure the correlation between bias-by projection and bias by neighbors, quantified as the percentage of the top 100 neighbors of each profession which were \emph{originally} biased-by-projection towards either of the genders. We find strong correlation of 0.734 (compared with 0.852 before), indicating that much of the bias-by-neighbors remains.\footnote{Note that if, for example, STEM-related words are originally biased towards men, the word ``chemist" after the projection may still be regarded as male-biased by neighbors, not because an inherent bias but due to its proximity to other \emph{originally} biased words (e.g. other STEM professions). 
}


    

    

    

\subsection{Fair Classification: Controlled Setup}
\label{sec:deepmoji}

We now evaluate using INLP with a deeper classifier, with the goal of achieving fair classification.
\\[0.2em]\textbf{Classifier bias measure: TPR-GAP.}
To measure the bias in a classifier, we follow \citet{biasbios} and use the TPR-GAP measure. This measure quantifies the bias in a classifier by considering the difference (GAP) in the True Positive Rate (TPR) between individuals with different protected attributes (e.g. gender/race).
The TPR-GAP is tightly related to the notion of fairness by equal opportunity \citep{equal-opportunities+odds}: a fair classifier is expected to show similar success in predicting the task label $Y$ for the two populations, when conditioned on the true class. Formally, for a binary protected attribute $z$ and a true class $y$, define:
\begin{align}
    TPR_{z,y} =& P[\hat{Y} = y | Z =  \textsl{z}, Y = y ]\\
    GAP_{z,y}^{TPR} =& TPR_{z,y} - TPR_{z',y}
\end{align}
where $Z$ is a random variable denoting binary protected attribute,  $z$ and $z'$ denote its two values, and $Y$, $\hat{Y}$ are random variables denoting the correct class and the predicted class, respectively.
\\[0.2em]\textbf{Experiment setup.}
We begin by experimenting with a controlled setup, where we control for the proportion of the protected attributes within each main-task class. We follow the setup of \citet{elazar2018} which used a twitter dataset, collected by \citet{blodgett2016demographic}, where each tweet is associated with ``race'' information and a sentiment which was determined by their belonging to some emoji group. 

Naturally, the correlation between the protected class labels and the main-class labels may influence the fairness of the model, as high correlation can encourage the model to condition on the protected attributes.  
We measure the TPR-GAP on predicting sentiment for the different race groups (African American English (AAE) speakers and Standard American English (SAE) speakers), with different imbalanced conditions, with and without application of our ``classifier debiasing'' procedure.

In all experiments, the dataset is overly balanced with respect to both sentiment and race (50k instances for each). We change only the \emph{proportion} of each race within each sentiment class (e.g., in the 0.7 condition, the ``happy'' sentiment class is composed of 70\% AAE / 30\% SAE, while the ``sad'' class is composed of 30\% AAE / 70\% SAE).

\begin{table}[t!]
\resizebox{\columnwidth}{!}{%
\begin{tabular}{c|cc|cc}

& \multicolumn{2}{c}{Sentiment} & 
\multicolumn{2}{c}{TPR-Gap} \\
Ratio & Original & INLP & Original & INLP \\ \hline
0.5 &	 0.76 & 	 0.75 &		 0.19 &		 0.16 \\
0.6 &	 0.78 & 	 0.74 &		 0.29 &		 0.22 \\
0.7 &	 0.81 & 	 0.66 &		 0.38 &		 0.24 \\
0.8 &	 0.84 & 	 0.67 &		 0.45 &		 0.15 \\
\end{tabular}
}
\caption{The Sentiment scores (in accuracy, higher is better) and TPR differences (lower is better) as a function of the ratio of tweets written by black individuals in the positive-sentiment class.}
\label{tbl:control-results}

\end{table}


Our classifier is based on the DeepMoji encoder \cite{felbo2017}, followed by a 1-hideen-layer MLP. The DeepMoji model was trained on millions of tweets in order to predict their emojis; a model which was proven to perform well on different classification tasks \cite{felbo2017}, but also encodes demographic information \cite{elazar2018}. We train this classifier to predict sentiment. We then follow the procedure in \S\ref{sec:deep-debiasing}: training a guarding function on the hidden layer of the MLP, and re-training the final linear layer on the guarded vectors. 
Table~\ref{tbl:control-results} presents the results. 

As expected the TPR-GAP grows as we increase the correlation between class labels and protected attributes. The accuracy grows as well. Applying our debiasing technique significantly reduced the TPR gap in all settings, although hurting more the main task accuracy in the highly-imbalanced setting. In Appendix \ref{sec:appendix-acc_tpr_inlp}, we give some more analysis on the balance between performance and TPR-Gap and show that one can control for this ratio, by using more iterations of INLP.

\subsection{Fair Classification: In the Wild}
\label{sec:bios}

We now evaluate the fair classification approach in a less artificial setting, measuring gender bias in biography classification, following the setup of \citet{biasbios}.

\begin{table}[t!]
\resizebox{\columnwidth}{!}{%
\begin{tabular}{ccccc}

                          &           & BoW & FastText & BERT \\ \hline
\multirow{2}{*}{Accuracy (profession)} & Original  & 78.2   & 78.1        & 80.9    \\
                          & +INLP & 80.1   & 73.0        & 75.2    \\ \hline
\multirow{2}{*}{$GAP_{male}^{TPR, RMS}$} & Original  & 0.203   & 0.184        & 0.184    \\
                          & +INLP & 0.124   &  0.089        &  0.095  \\ \hline
\end{tabular}
}
\caption{Fair classification on the Biographies corpus.}
\label{tbl:bios-results}

\end{table}

\begin{figure}[t!]
    \centering
    \includegraphics[width=0.8\columnwidth]{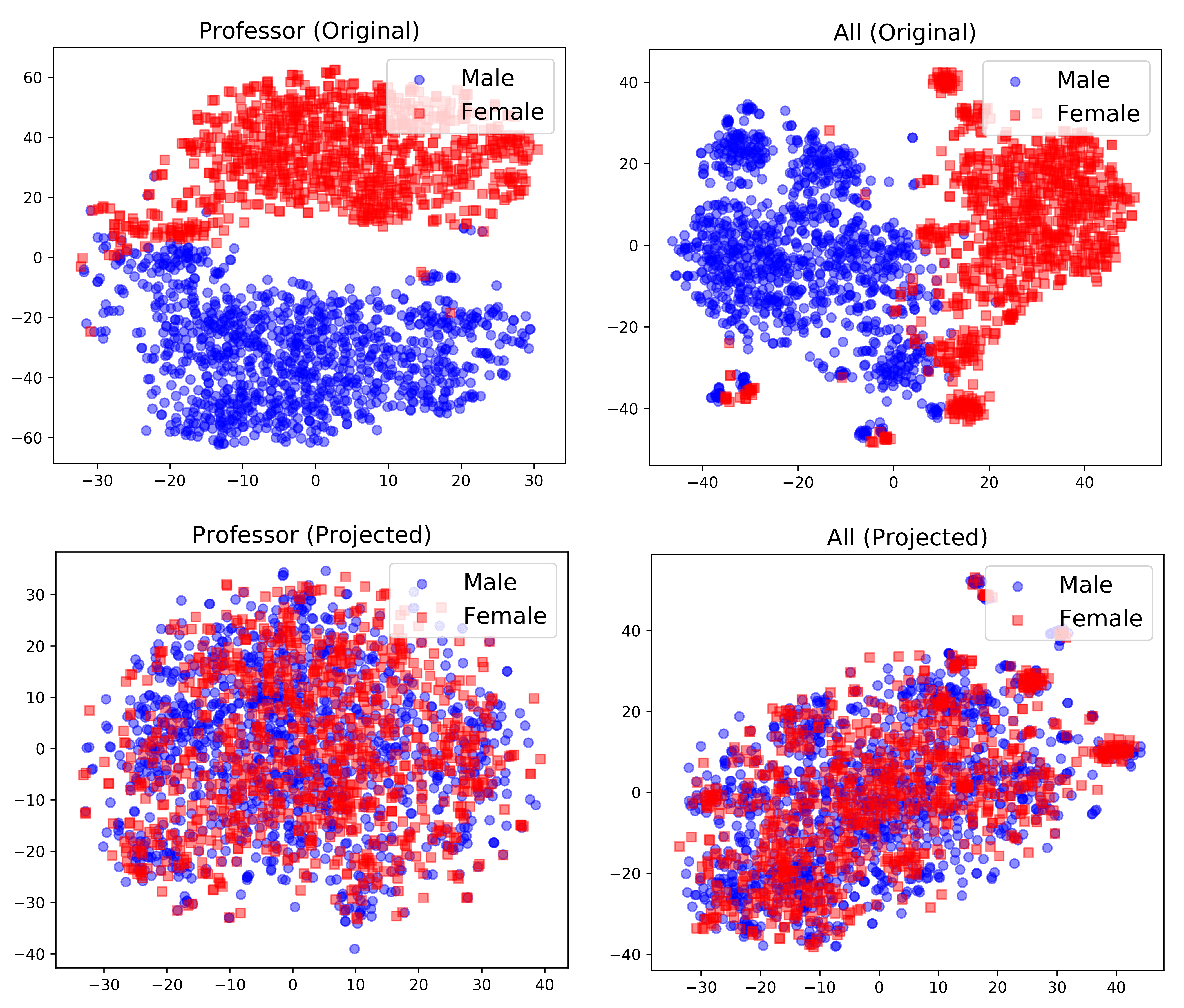}
    \caption{t-SNE projection of BERT representations for the profession ``professor" (left) and for a random sample of all professions (right), before and after the projection.}
    \label{fig:bert-change-tsne}
    
\end{figure}

\begin{figure}[t!]
\centering
\subfloat{
\includegraphics[width=0.5\columnwidth]{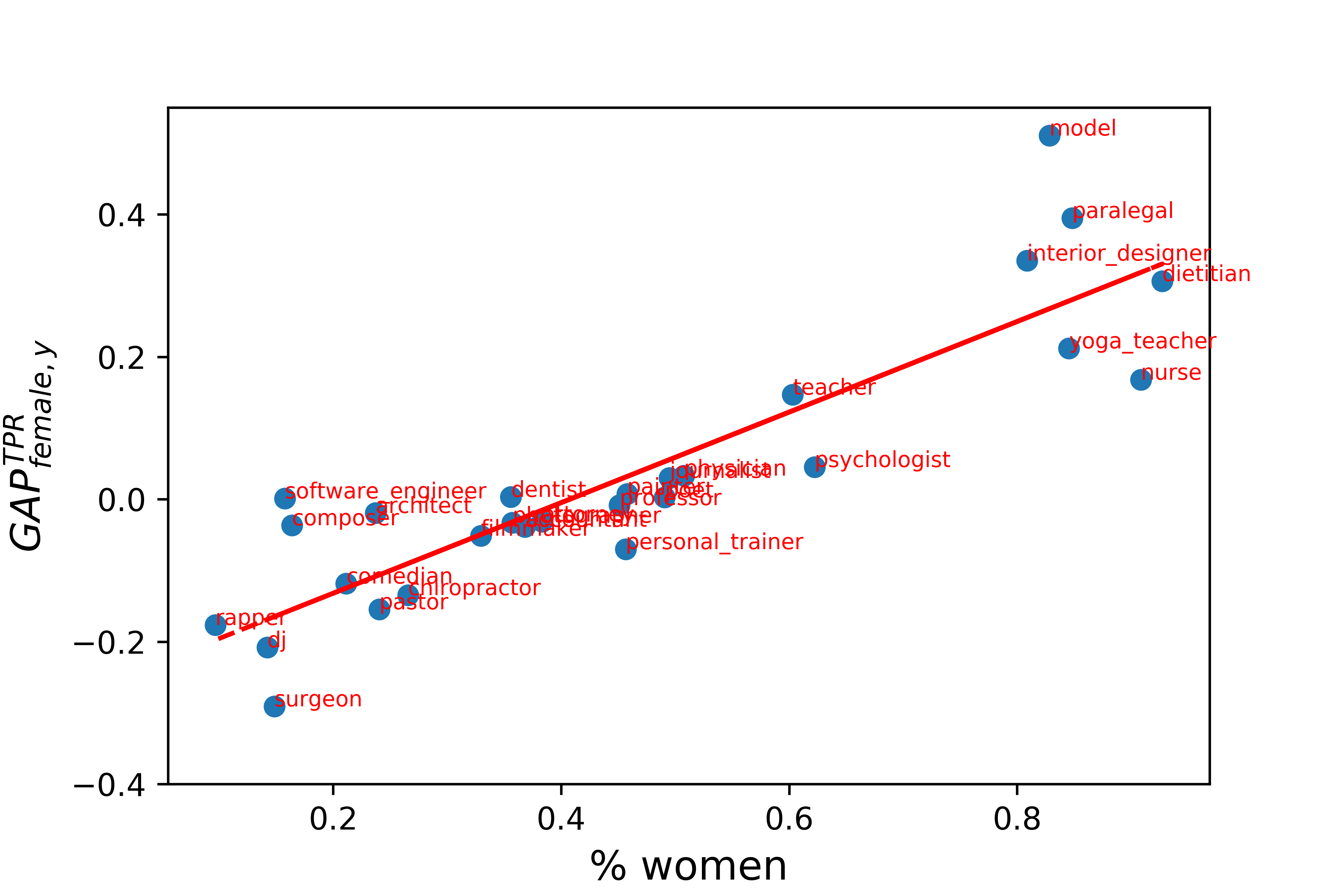}
\label{fig:correlations-bert-biased}
}
\subfloat{
\includegraphics[width=0.5\columnwidth]{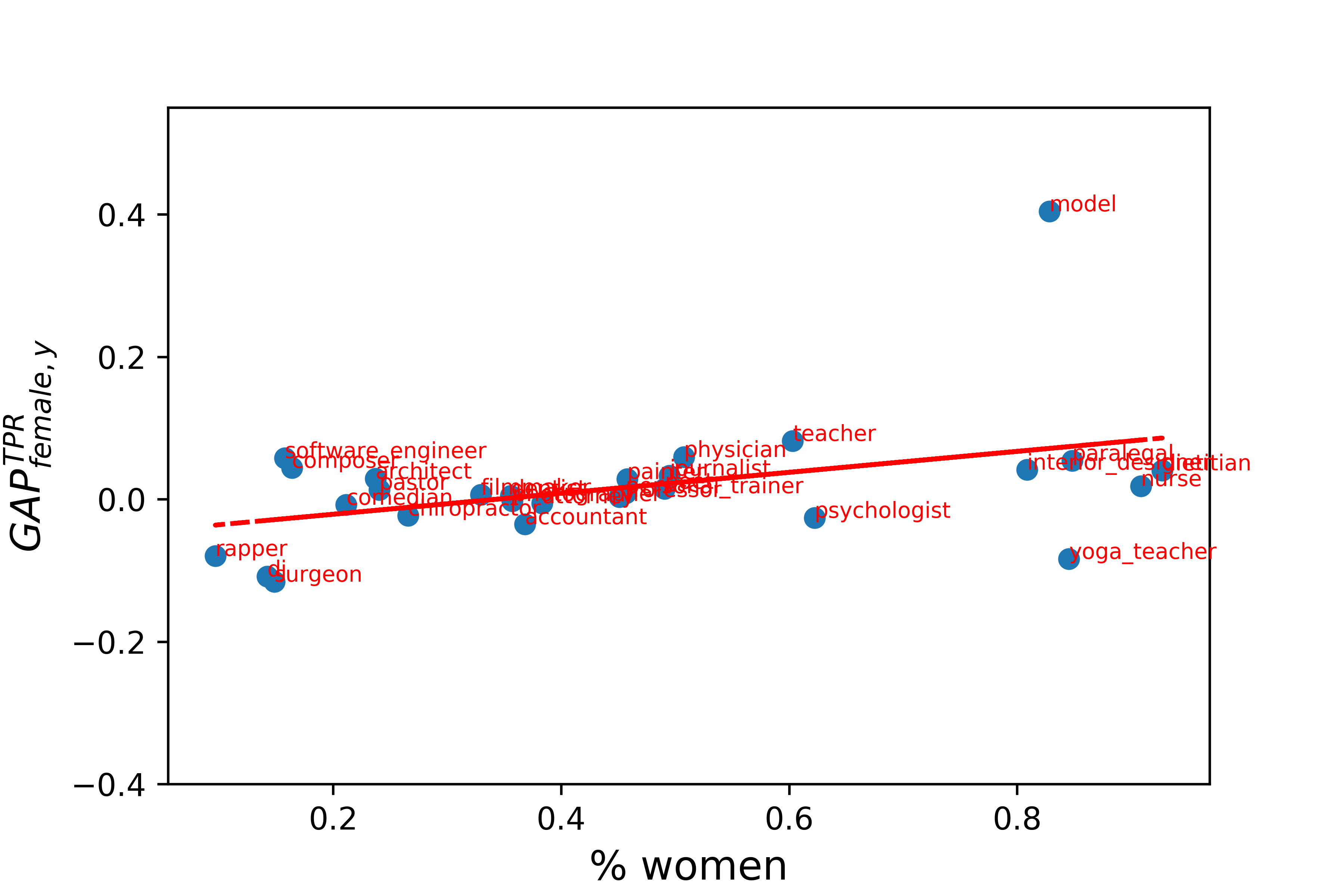}
\label{fig:correlations-bert}
}
\caption{The correlation between $GAP_{female,y}^{TPR}$ and the relative proportion of women in profession $y$, for BERT representation, before (left; R=0.883) and after (right; R=0.470) the projection.}
\label{fig:bert-change-tpr}
\end{figure}

They scraped the web and collected a dataset of short biographies, annotated by gender and profession. They trained logistic regression classifiers to predict the profession of the biography's subject based on three different input representation: bag-of-words (BOW), bag of word-vectors (BWV), and RNN based representation. 
We repeat 
their experiments, using INLP for rendering the classifier oblivious of gender.  
\\[0.2em]\textbf{Setup.} Our data contains 393,423 biographies.\footnote{The original dataset had 399,000 examples, but 5,557 biographies were no longer available on the web.} We follow the train:dev:test split of \citet{biasbios}, resulting in 255,710 training examples (65\%), 39,369 development examples (10\%) and 98,344 (25\%) test examples. The dataset has 28 classes (professions), which we predict using a multiclass logistic classifier (in a one-vs-all setting). We consider three input representations: BOW, BWV and BERT \citep{bert} based classification. In BOW, we represent each biography as the sum of one-hot vectors, each representing one word in the vocabulary. In the BWV representation, we sum the FastText token representations \citep{fasttext} of the words in the biography. In BERT representation, we represent each biography as the last hidden state of BERT over the $CLS$ token. Each of these representations is then fed into the logistic classifier to get final prediction. We do not finetune FastText or BERT.  

We run INLP with scikit-learn \citet{scikit-learn} linear classifiers. We use 100 logistic classifiers for BOW, 150 linear SVM classifiers for BWV, and 300 linear SVM classifiers for BERT.  
\\[0.2em]\textbf{Bias measure.} We use the TPR-GAP measure for each profession. 
Following \citet{biasbios2}, we also calculate the root-mean square of $GAP_{g,y}^{TPR}$ over all professions $y$, to get a single per-gender bias score:
\begin{equation}
    GAP_{g}^{TPR, RMS} = \sqrt{\frac{1}{|C|} \sum_{y \in C} (GAP_{g,y}^{TPR})^2}
\end{equation}
where $C$ is the set of all labels (professions). 

\citet{biasbios} have shown that $GAP_{g,y}^{TPR}$ strongly correlates with the percentage of women in profession $y$, indicating that the true positive rate of the model is influenced by gender.

\subsubsection{Results}
\label{sec:bios-results}
\paragraph{Main results} The results are summarized in Table \ref{tbl:bios-results}. INLP moderately changes main-task accuracy, with a 1.9\% increase in BOW, a 5.1\% decrease in performance in BWV and a 5.51\% decrease in BERT. $GAP_{g}^{TPR, RMS}$ is significantly decreased, indicating that on average, the true positive rate of the classifiers for male and female become closer: in BOW representation, from 0.203 to 0.124 (a 38.91\% decrease); in BWV, from 0.184 to 0.089 (a 51.6\% decrease); and in BERT, from 0.184 to 0.095 (a 48.36\% decrease).  We measure the correlation between $GAP_{y, female}^{TPR}$ for each profession $y$, and the percentage of biographies of women in that profession. In BOW representation, the correlation decreases from 0.894 prior to INLP to 0.670 after it (a 33.4\% decrease). In BWV representation, the correlation decreases from 0.896 prior to INLP to 0.425 after it (a 52.5\% decrease). In BERT representation, the correlation decreases from 0.883 prior to INLP to 0.470 following it (a 46.7\% decreases; Figure \ref{fig:correlations-bert}). \citet{biasbios} report a correlation of 0.71 for BWV representations when using a ``scrubbed" version of the biographies, with all pronouns and names removed. INLP significantly outperforms this baseline, while maintaining all explicit gender markers in the input. 
\\[0.2em]\textbf{Analysis.} How does imposing fairness influence the importance the logistic classifier attribute to different words in the biography? We take advantage of the BOW representation and visualize which features (words) influence each prediction (profession), before and after the projection.
According to Algorithm \ref{algo:projection}, to debias an input $x$, we multiply $W(Px)$. Equivalently, we can first multiply $W$ by $P$ to get a ``debiased'' weight matrix $W'$. 
We begin by testing how much the debiased weights of words that are considered to be biased 
were changed during the debiasing, compared to random vocabulary words. We compare the relative change before and after the projection of these words, for every occupation. 
Biased words undergo an average relative change of x1.23 compared to the average change of the entire vocabulary, demonstrating that biased words indeed change more.
The per-profession breakout is available in Figure \ref{fig:rel-change} in Appendix \S\ref{sec:appendix-bios-bag-of-words}.\\
\indent Next, we test the words that were changed the most during the INLP process. We compare the weight difference before and after the projection. We sort each profession words by weight, and average their location index for each professions. Many words indeed seem gender specific (e.g. \textit{ms.}, \textit{mr.}, \textit{his}, \textit{her}, which appears in locations 1, 2, 3 and 4 respectively), but some seem unrelated, perhaps due to spurious correlations in the data.
The complete list is available in Table \ref{table:word_change} in the Appendix \S \ref{sec:appendix-bios-bag-of-words}; an analogous analysis for the FastText representation is available at Appendix \S \ref{sec:appendix-bios-bag-of-vectors}.

\section{Limitations}
\label{sec:limitations}
A limitation of our method when used in the context of fairness is that, like other learning approaches, it depends on the data $X$,$Z$ that is fed to it, and works under the assumption that the training data is sufficiently large and is sampled i.i.d from the same distribution as the test data. This condition is hard to achieve in practice, and failure to provide sufficiently representative training data may lead to biased classifications even after its application. Like other methods, there are no magic guarantees, and the burden of verification remains on the user. It is also important to remember that the method is designed to achieve a very specific sense of protection: removal of linear information regarding a protected attribute. While it may correlate with fairness measures such as demographic parity, it is not designed to ensure them. Finally, it is designed to be fed to a linear decoder, and the attributes are not protected under non-linear classifiers.

\section{Conclusion} We present a novel method for removing linearly-represented information from neural representations. We focus on bias and fairness as case studies, and demonstrate that across increasingly complex settings, our method is capable of attenuating societal biases that are expressed in representations learned from data. 

While this work focuses on societal bias and fairness, Iterative Nullspace Projection has broader possible use-cases, and can be utilized to remove specific components from a representation, in a controlled and deterministic manner. This method can be applicable for other end goals, such as style-transfer, disentanglement of neural representations and increasing their interpretability. We aim to explore those directions in a future work.  



\section*{Acknowledgements}
We thank Jacob Goldberger and Jonathan Berant for fruitful discussions. 
This project received funding from the Europoean Research Council (ERC) under the Europoean Union's Horizon 2020 research and innovation programme, grant agreement No. 802774 (iEXTRACT).

\bibliographystyle{acl_natbib}
\bibliography{main}

\newpage
\clearpage

\appendix
\section{Appendix}
\label{appendix}

\setcounter{figure}{0}
\setcounter{table}{0}

\subsection{INLP Guarantees}
\label{INLP-proofs}


In this section, we prove, for the binary case, an orthogonality property for INLP classifiers: each two classifiers $w_i$ and $w_j$ from two iterations steps $i$ and $j$ are orthogonal (Lemma \ref{lem:orthogonality}). Several useful properties of the matrix $P$ that is returned from INLP emerge as a direct result of orthogonality: the product of the projection matrices calculated in the different INLP steps is commutative (Corollary \ref{lem:comm}); P is a valid projection (Corollary \ref{lem:is-projection}); and P projects to a subspace which is the intersection of the nullspaces of all INLP classifiers $ N(w_1) \cap N(w_2) \cap \dots \cap N(w_n)$ (Corollary \ref{lem:is-projection-to-intersection}). Furthermore, we bound the influence of $P$ on the structure of the representation space, demonstrating that its impact is limited only to those parts of the vectors that encode the protected attribute (Lemma \ref{lem:distances}).

We prove those properties for two consecutive projection matrices $P_1$ and $P_2$ from two consecutive iterations of Algorithm \ref{algo:projection}, presented below in \ref{two-iteration}. The general property follows by induction. 
\begin{enumerate}
  \item $w_1 = \operatorname*{argmin}_w \mathcal{L}(w; X; Z)$
  \item $P_1:=P_{N(w_1)}$= GetProjectionMatrix($N(w_1$))
  \item $X' = P_1x$
  \item $w_2 = \operatorname*{argmin}_w \mathcal{L}(w; X'=P_1X; Z)$
  \item $P_2:=P_{N(w_2)}$= GetProjectionMatrix($N(w_2$))
  \label{two-iteration}
\end{enumerate}

\textbf{INLP Projects to the Intersection of Nullspaces.}

\begin{lemma}
\label{lem:orthogonality}
if $w_2$ is initialized as the zero vector and trained with SGD, and the loss $\mathcal{L}$ is convex, then $w_2$ is orthogonal to $w_1$, that is, $w_1 \cdot w_2 = 0$.
\end{lemma}

\begin{proof}
In line 4 of the algorithm, we calculate $w_2 = \operatorname*{argmin}_w \mathcal{L}(w; X'=P_1X; Z)$. For a convex $\mathcal{L}$ and a linear model $w$, it holds that the gradient with respect to $w$ is a linear function of $x$: $\nabla_w{}\mathcal{L}(x_i)=\alpha_ix_i$ for some scalar $\alpha_i$. It follows that after $t$ stochastic SGD steps, $w_2^t$ is a linear combination of input vectors $x_1, \dots, x_i, \dots, x_t$. Since we constrain the optimization to $x_i \in N(w_1)$, and considering that fact the nullspace is closed under addition, at each step $t$ in the optimization it holds that $w_2^t \in N(w_1)$. In particular, this also holds for the optimal $w_2^*$ \footnote{If we performed proper dimensionality reduction at stage 3 -- i.e., not only zeroing some directions, but completely removing them -- the optimization in 4 would have a unique solution, as the input would not be rank-deficient. Then, we could use an alternative construction that relies on the Representer theorem, which allows expressing $w_2$ as a weighted sum of the inputs: $w_2 = \sum_{x_i \in X'=P_1X}^{} \alpha_ix_i$, for some scalars $\alpha_i$. As each $x_i$ is inside the nullspace, so is any linear combinations of them, and in particular $w_2$.}. \end{proof}
We proceed to prove commutativity based on this property.


\begin{corollary}
\label{lem:comm}
$P_1P_2 = P_2P_1$
\end{corollary}

\begin{proof}
By Lemma \ref{lem:orthogonality}, $w_1 \cdot w_2 = 0$, so $P_{R(w_1)}P_{R(w_2)} = P_{R(w_2)}P_{R(w_1)} = 0$, where $P_{R(w_i)}$ is the projection matrix on the row-space of $w_i$. We rely on the relation $P_{N(w_i)} = I - P_{R(w_i)}$ and write:

$P_1P_2 = (I-P_{R(w_1)})(I-P_{R(w_2)}) = I - P_{R(w_1)} - P_{R(w_2)} - P_{R(w_1)}P_{R(w_2)} = I - P_{R(w_1)} - P_{R(w_2)}$.

Similarly,

$P_2P_1 = (I-P_{R(w_2)})(I-P_{R(w_1)}) = I - P_{R(w_2)} - P_{R(w_1)} - P_{R(w_2)}P_{R(w_1)} = I - P_{R(w_1)} - P_{R(w_2)}$, which completes the proof.
\end{proof}

\begin{corollary}
\label{lem:is-projection}
$P=P_2P_1$ is a projection, that is, $P^2 = P$.
\end{corollary}

\begin{proof}
$P^2 = (P_2P_1)^2 = P_2P_1P_2P_1 \stackrel{*}{=} P_2P_2P_1 P_1 = P_2 ^2P_1 ^2  \stackrel{**}{=} P_2P_1=P$, where $*$ follows from Corollary \ref{lem:comm} and $**$ follows from $P_1$ and $P_2$ being projections. 
\end{proof}

\begin{corollary}
\label{lem:is-projection-to-intersection}
$P_2P_1$ is a projection onto $N(w_1) \cap N(w_2)$.
\end{corollary}

\begin{proof}
Let $x \in \R^n$. $P_2(P_1x) \in N(w_2)$, as $P_2$ is the projection matrix to $N(w_2)$. Similarly, $P_2(P_1x)=P_1(P_2x) \in N(w_1)$, so $Px \in N(w_1) \cap N(w_2)$. Conversely, let $x \in N(w_1) \cap N(w_2)$. Then $P_1x = x = P_2x$, so $Px = P_2P_1x = P_2x = x$, so $x$ is mapped by $P$ to $N(w_1) \cap N(w_2)$.

\end{proof}

Note that in practice, we enforce Corollary \ref{lem:is-projection-to-intersection} by using the projection Equation \ref{ben-israel-formula} (section \ref{implementaton}). As such, the matrix $P$ that is returned from Algorithm \ref{algo:projection} is a valid projection matrix to the intersection of the nullspaces even if the the conditions in Lemma \ref{lem:orthogonality} do not hold, e.g. when $ \mathcal{L}$ is nonconvex or $w_2$ is not initialized as the zero vector.

\textbf{INLP Approximately Preserves Distances.}

While the projection operations removes the protected information from the representations, ostensibly it could have had a detrimental impact on the structure of the representations space: as a trivial example, the zero matrix $\boldsymbol{O}$ is another operator that removes the protected information, but at a price of collapsing the entire space into the zero vector. The following lemma demonstrate this is not the case. The projection minimally damages the structure of the representation space, as measured by distances between arbitrary vectors: the change in (squared) distance between $x,x' \in \R^n$ is bounded by the difference between the ``gender components" of $x$ and $x'$.

\label{distance-preservation}

\begin{lemma}
\label{lem:distances}
Let $\vv{w} \in \R^n$ be a unit gender direction found in one INLP iteration, and let $x,x' \in \R^n$ be arbitrary input vectors. Let $P:=P_{N(\vv{w})}$ be the nullspace projection matrix corresponding to $\vv{w}$.  Let  $d(x,x') := ||x-x'||$ and $d^{P}(x,x') := ||Px-Px'||$ be the distances between $x,x'$ before and after the projection, respectively. Then the following holds:

$( d(x,x') -d^{P}(x,x'))^2 \leq (x \vv{w} - x' \vv{w})^2$
\end{lemma}

\begin{proof}

\textbf{notation}: we denote the $i$th entry of a vector $x$ by  $x_i$. 

Since $\vv{w}$ is the parameter vector of a gender classifier, a point $x \in \R^n$ can be classified to a gender $g \in \{0,1 \}$ according to the sign of the dot product $x\vv{w}$. Note that in the binary case, the nullspace projection matrix $P$ is given by

\begin{equation}
    P = I - \vv{w}\vv{w}^T
\end{equation}

Where $\vv{w}\vv{w}^T$ is the outer product. By definition, if $\vv{w}$ is in the direction of one of the axes, say without loss of generality the first axis, such that $\vv{w} = [1, 0, ..., 0]$, then the following holds:

\begin{equation}
\vv{w}\vv{w}^T = \begin{bmatrix} 
    1 & 0 & \dots & 0 \\
    \vdots & 0  & \\
    &  & \ddots & \\
    0 &    \dots  &  & 0 
    \end{bmatrix}
\end{equation}

Such that $\vv{w}\vv{w}^T$ is the zero matrix except its $(1,1)$ entry, and then $P$ is simplified to 


\begin{equation}
P = \begin{bmatrix} 
    0 & 0 & \dots & 0 \\
    \vdots & 1 & 0 \dots & 0 \\
    0 &  & \ddots & \\
    0 &    \dots  &  & 1 
    \end{bmatrix}
\label{p-simflieid}
\end{equation}

I.e, the unit matrix, except of a zero in the $(1,1)$ position. Hence, the projection operator $Px$ keeps $x$ intact, apart from zeroing the first coordinate $x_1$. We will take advantage of this property, and rotate the axes such that $\vv{w}$ is the direction of the first axis. We will show that the results we derive this way still apply to the original axes system.

Let $R$ be a rotation matrix, such that after the rotation, the first coordinate of $Rx$ is aligned with $\vv{w}$:

\begin{equation}
  {(Rx)}_1 =x\vv{w}
\end{equation}

One can always find such rotation of the axes. Let $x' \in \R^n$ be another point in the same space. Given the original squared distance between $x$ and $x'$:

\begin{equation}
    d(x,x') = ||x-x'||^2
\end{equation}

Our goal is to bound the squared distance between the projected points in the new coordinate system:

\begin{equation}
    d^{P, R}(x,x') := ||[P]_RRx-[P]_RRx'||^2 
\end{equation}

Where $[P]_R$ denotes the projection matrix $P$ in the rotated coordinate system, which takes the form \ref{p-simflieid}.

Note that $R$, being a rotation matrix, is orthogonal. By a known result in linear algebra, multiplication by orthogonal matrices preserves dot product and distances. That means that the distance is the same before and after the rotation: $d^{P,R}(x,x') = d^{P}(x,x')$, so we can safely bound $d^{P,R}(x,x')$ and the same bound would hold in the original coordinate system.\\

\newpage
\clearpage

By \ref{p-simflieid}, 

\begin{align}
    &d^{P, R}(x,x') \notag \\ \label{dist-projected}
    &=\sqrt{(0 - 0)^2 + \sum_{i=2}^n ([Rx]_i - [Rx']_i)^2 } 
    \\  &= \sqrt{d(x,x')^2 - ([Rx]_1-[Rx']_1)^2} \notag
\end{align}

Note that in general it holds that for any $a \geq b \geq 0$

\begin{align}
    \sqrt{a - b} &= \sqrt{a +b - 2b} = \sqrt{a + b - 2\sqrt{b}\sqrt{b}} \label{sqrt-diff-squares} \\ \notag{}
    & \geq \sqrt{a + b - 2\sqrt{a}\sqrt{b}} = \sqrt{(\sqrt{a}-\sqrt{b})^2}  \\ \notag{}
    &= \sqrt{a} - \sqrt{b} 
\end{align}

Combining \ref{sqrt-diff-squares} with \ref{dist-projected} when taking $a = d(x,x')^2, b = ([Rx]_1-[Rx']_1)^2$ we get:

\begin{equation}
    d^{P,R}(x,x') \geq d(x,x') - |([Rx]_1-[Rx']_1)|
    \label{greater-equal}
\end{equation}

From \ref{dist-projected} one can also trivially get 

\begin{align}
     d^{P,R}(x,x') &= \sqrt{d(x,x')^2 - ([Rx]_1-[Rx']_1)^2)} \label{less-equal} \\ \notag
     &\leq  \sqrt{d(x,x')^2} = d(x,x')
\end{align}
 
 Combining \ref{less-equal} and \ref{greater-equal} we finally get:
 
 \begin{align}
      d(x,x') - |([Rx]_1-[Rx']_1)| &\leq  d^{P,R}(x,x') \\ \notag &< d(x,x')
 \end{align}
 
 Or, equivalently, after subtracting $d(x,x')$ from all elements and multiplying by -1:
 
 \begin{align*}
     |([Rx]_1-[Rx']_1)| &= |x \vv{w} - x'\vv{w}| \notag \\ & \geq d(x,x') - d^{P,R}(x,x')  \geq 0 
 \end{align*}
 
 So 
 
 \begin{align*}
     &( d(x,x') - d^{P,R}(x,x'))^2  \\
     & \leq ([Rx]_1-[Rx']_1)^2  = (x \vv{w} - x' \vv{w})^2
 \end{align*}
 
 
 Note that this result has a clear interpretation: the difference between the distance of the projected $x,x'$ and the distance of the original $x, x'$ is bounded by the difference of $x$ and $x'$ in the gender direction $\vv{w}$. In particular, if $x$ and $x'$ are equally male-biased, their distance would not change at all; if $x$ is very male-biased and $x'$ is very female-biased, the projection would significantly alter the distance between them. 
 
 \end{proof}

\newpage
\clearpage

\newpage
\clearpage

\subsection{Influence on Local Neighbors in Glove Space}
\label{sec:appendix-neighbors-examples}

\begin{table}[H]
\begin{tabular}{lll}
\toprule
         Word &                      Neighbors before &                            Neighbors after \\
\midrule
        order &            orders, ordering, purchase &                  orders, ordering, ordered \\
        crack &               keygen, cracks, torrent &                      keygen, cracks, warez \\
   craigslist &                  ebay, craiglist, ads &                 ebay, craiglist, freecycle \\
  populations &      population, species, communities &              population, species, habitats \\
         epub &                      ebook, mobi, pdf &                        mobi, ebook, kindle \\
       finals &    semifinals, playoffs, championship &       semifinals, semifinal, quarterfinals \\
    installed &     install, installing, installation &              install, installing, installs \\
 identifiable &       disclose, identify, identifying &             disclose, pii, distinguishable \\
  photographs &            photograph, photos, images &                 photograph, images, photos \\
           ta &                            si, tu, ti &                             que, bien, ele \\
        couch &                    sofa, sitting, bed &                    sofa, couches, loveseat \\
       cooler &              coolers, cooling, warmer &                   coolers, cooling, warmer \\
        becky &                  debbie, kathy, julie &                         debbie, steph, jen \\
  appreciated &           appreciate, greatly, thanks &                 appreciate, muchly, thanks \\
  negotiation &  negotiating, negotiations, mediation &       negotiating, negotiations, mediation \\
      initial &          subsequent, prior, following &                 intial, inital, subsequent \\
        chloe &                  chanel, emma, lauren &                    chloé, chanel, handbags \\
     filipino &          pinoy, filipinos, philippine &                  filipinos, pinoy, tagalog \\
      relying &                  rely, relied, relies &                       rely, relied, relies \\
    perpetual &       eternal, continual, irrevocable &          irrevocable, datejust, perpetuity \\
      himself &                     him, herself, his &                       herself, oneself, he \\
      seaside &         beach, beachside, picturesque &               beachside, idyllic, seafront \\
      measure &         measures, measuring, measured &              measures, measuring, measured \\
    yorkshire &      staffordshire, leeds, lancashire &           staffordshire, dales, lancashire \\
  merchandise &                 goods, items, apparel &                  goods, items, merchandize \\
          sub &                          subs, k, def &                          subs, subbed, svs \\
        tones &                     tone, hues, muted &                    tone, polyphonic, muted \\
    therapist &     therapists, psychologist, therapy &  therapists, physiotherapist, psychologist \\
       leaned &               sighed, smiled, glanced &                     leant, leaning, sighed \\
          tho &                        nnd, cuz, tlie &                            nnd, tlio, tlie \\
      lawyers &           attorneys, lawyer, attorney &                attorneys, lawyer, attorney \\
      compile &         compiling, compiler, compiles &              compiling, compiler, compiles \\
        chord &          chords, progressions, guitar &             chords, progressions, voicings \\
         aims &                    aim, aimed, aiming &                         aim, aimed, aiming \\
       ensure &             ensuring, assure, ensures &                  ensuring, ensures, assure \\
    aerospace &     aviation, engineering, automotive &        aeronautics, aviation, aeronautical \\
    clubhouse &           pool, playground, amenities &               clubhouses, pool, playground \\
      locking &                    lock, locks, latch &                         lock, locks, latch \\
        reign &               reigns, emperor, throne &                   reigns, reigned, emperor \\
   vulnerable &        susceptible, fragile, affected &            susceptible, vunerable, fragile \\
\bottomrule
\end{tabular}

\caption{3-nearest words before and after the INLP projection}
\label{table:nearest_words}
\end{table}

\newpage
\clearpage

Table \ref{table:nearest_words} above presents the results of word-embeddings similarity test mentioned in \ref{sec:debiasing-word-embeddings}. This table lists the top 3-nearest neighbors of sampled words from GloVe, before and after the INLP process. It is evident that INLP does not alter the neighbors of the random sample in a detrimental way.


\subsection{Quantitative Influence of Gender Debiasing on Glove Embeddings}
\label{sec:appendix-quantitative_emb}
In Appendix \ref{sec:appendix-neighbors-examples} we provide a sample of words to qualitatively evaluate the influence of INLP on semantic similarity in Glove word embeddings (Section \ref{sec:debiasing-word-embeddings}). We observe minimal change to the nearest neighbors. To complement this measure, we use a quantitative measure: measuring performance on established word-similarity tests, for the original Glove embeddings, and for the debiased ones. Those tests measure correlation between cosine similarity in embedding space and human judgements of similarity. 
Concretely, we test the embeddings similarities using three dataset, which contain four similarity tests that measure similarity or relatedness between words.
We use the following datasets: SimLex999 \cite{hill2015simlex}, WordSim353 \cite{agirre2009study} which contain two evaluations, on words similarity and relatedness and finally on Mturk-771 \cite{halawi2012large}.

The test sets are composed of word pairs, where each pair was annotated by humans to give a similarity or relatedness score.
To evaluate a model against such data, each pair is given a score (in the case of word embedding, cosine similarity) and then we calculate Spearman correlation between all the score pairs.
The results on the regular Glove embeddings before and after the gender debiasing are presented in Table \ref{tbl:word_emb_sim}.
We observe a major improvements across all evaluation sets after the projection: between 0.044 to 0.116 points.

This major difference in performance is rather surprising. It is not clear how to interpret the positive influence on correlation with human judgements. This puzzle is further compounded by the fact the projection reduces the rank of the embedding spaces, and by definition induces loss of information.
We hypothesize that many of the words in the embedding space contain a significant gender component, which is not correlated with humans judgements of similarity. While intriguing, testing this hypothesis is beyond the scope of this work, and we leave the more rigorous answer to a future work.

\subsection{Influence on Local Neighbors of Surnames Representations in Glove Space}
\label{sec:appendix-neighbors-examples-names}

\begin{table}[H]
\centering
\resizebox{\columnwidth}{!}{%
\begin{tabular}{lll}
\toprule
      Word &           Neighbors before &                         Neighbors after \\
\midrule
      ruth &    helen, esther, margaret &                  etting, esther, gehrig \\
 charlotte &       raleigh, nc, atlanta &                 raleigh, greensboro, nc \\
   abigail &       hannah, lydia, eliza &                   hannah, phebe, josiah \\
    sophie &         julia, marie, lucy &                  moone, bextor, marceau \\
   nichole &    nicole, kimberly, kayla &                    nicole, mya, heiress \\
      emma &         emily, lucy, sarah &                    grint, frain, watson \\
    olivia &         emma, rachel, kate &                    munn, thirlby, wilde \\
       ava &      devine, zoe, isabella &       viticultural, devine, appellation \\
  isabella &  sophia, josephine, isabel &           rossellini, beeton, ferdinand \\
    sophia &         anna, lydia, julia &               hagia, antipolis, topkapi \\
       mia &         bella, mamma, mama &                     bangg, mamma, culpa \\
    amelia &  earhart, louisa, caroline &            earhart, fernandina, bedelia \\
     james &      john, william, thomas &                jassie, nightfire, perse \\
      john &       james, william, paul &                deere, scatman, betjeman \\
    robert &    richard, william, james &  pattinson, mccammon, blacksportsonline \\
   michael &         david, mike, brian &               micheal, franti, moorcock \\
   william &       henry, edward, james &                      edward, henry, sir \\
     david &  stephen, richard, michael &                  bisbal, magen, sylvian \\
   richard &     robert, william, david &            clayderman, brautigan, rorty \\
    joseph &   francis, charles, thomas &               joesph, dreamcoat, abboud \\
    thomas &       james, william, john &                    szasz, deshaun, tomy \\
     ariel &      sharon, alexis, hanna &           peterpan, mermaid, cinderella \\
      mike &         brian, chris, dave &                mignola, birbiglia, dave \\
\bottomrule
\end{tabular}
}
\caption{3-nearest words before and after the INLP projection, for surenames}
\label{table:nearest_words-names}
\end{table}

The results in Table \ref{table:nearest_words} suggest that, as expected, the projection has little influence on the lexical semantics of unbiased words, as measured by their closest neighbors in embedding space. But how does the projection influence inherently gendered words? Table \ref{table:nearest_words-names} contains the closest-neighbors to the Glove representations of gendered surnames, before and after the projection. 
We observe an interesting tendency to move from neighbors which are other gendered surnames, towards family names, which are by definition gender-neutral (for instance, the closest neighbor of ``Robert" changes from ``Richard" to ``Pattinson"). Another interesting tendency is to move towards place names bearing a connection to that surnames (For instance, the closest neighbor of ``Sophia" changs to ``Hagia"). At the same time, some gendered surnames remain close neighbors even after the projection.

\subsection{Performance and ``Fair Classification'' as a Function of INLP Iterations}
\label{sec:appendix-acc_tpr_inlp}

\begin{table}[ht]
\resizebox{\columnwidth}{!}{%
\begin{tabular}{c|cc}

Eval  & Before  & After \\ \hline
SimLex999        & 0.373 & \textbf{0.489} \\
WordSim353 - Sim & 0.695 & \textbf{0.799} \\
WordSim353 - Rel & 0.599 & \textbf{0.698} \\
Mturk-771        & 0.684 & \textbf{0.728} \\
\end{tabular}

}
\caption{Word similarity scores on Glove embeddings, before and after INLP. The scores are the Spearman correlation coefficient between the similarity scores.}
\label{tbl:word_emb_sim}

\end{table}

In Section \ref{sec:deepmoji} where we compare the accuracy and TPR-Gap before and after using INLP for a certain amount of iterations.
The number of iterations chosen is somehow arbitrary, but we emphasize that this can be controlled for as the number of iterations used with INLP. By sacrificing the main task performance, one can improve the TPR-Gap of their model.
In Figure \ref{fig:acc-tpr} we detail these trade-offs for the $0.8$ ratio, where the original TPR-Gap originally is the highest.

We note that the performance is minimally damaged for the first 180 iterations, while the TPR-Gap improves greatly, after-which, both metric account for larger drops.
Using this trade-off, one can decide how much performance they are willing to sacrifice in order to get a less biased model.

\begin{figure}[H]
    \centering
    \includegraphics[width=1.0\columnwidth]{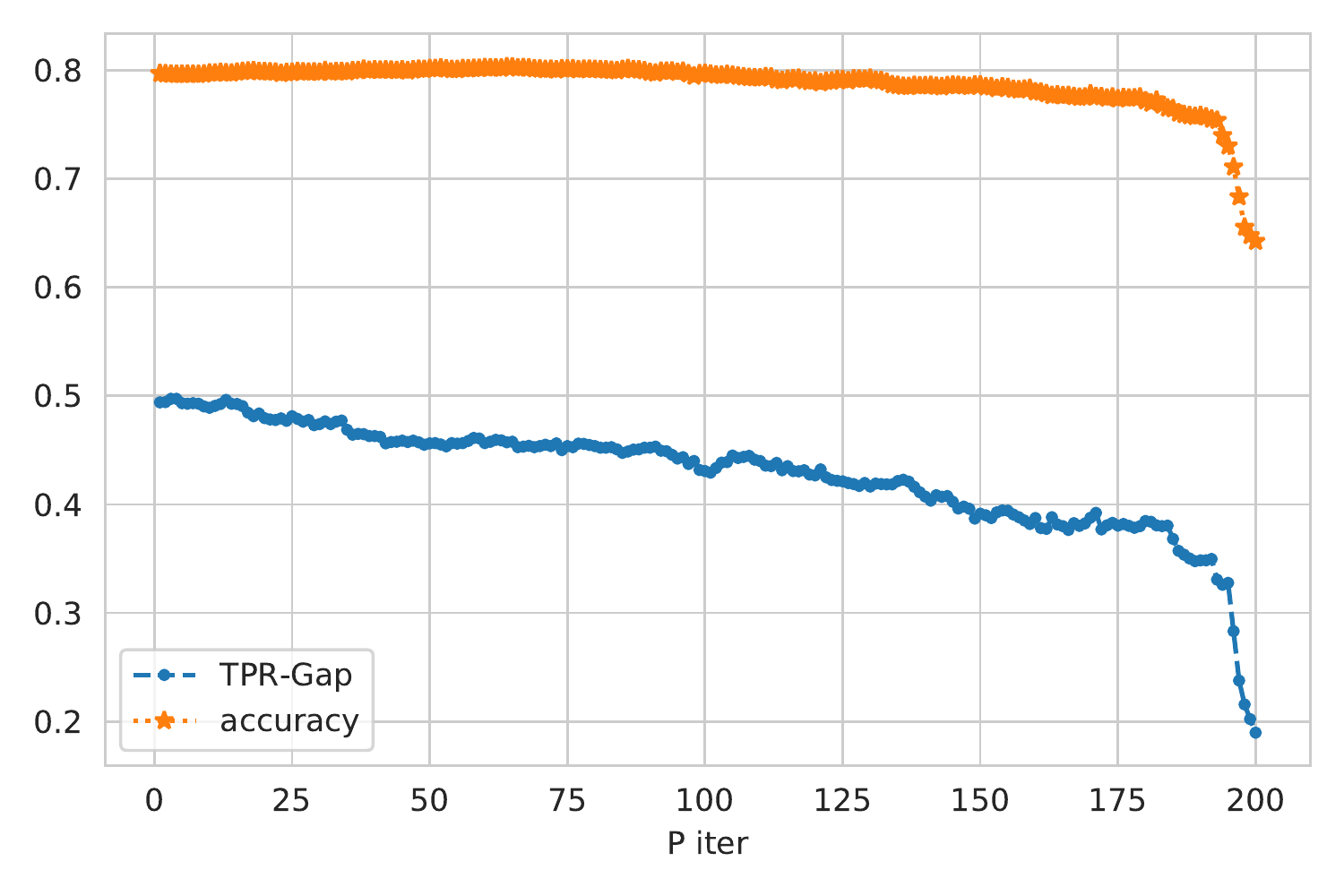}
    \caption{}
    \label{fig:acc-tpr}
\end{figure}


\subsection{Biographies dataset: Words Most-Associated with Gender}
\subsubsection{Bag-of-Words Model}
\label{sec:appendix-bios-bag-of-words}
In this section, we present the raw results of the experiment aimed to assess the influence of INLP on specific words, under the bag-of-words model, for the biographies experiments (Section \ref{sec:bios-results}).

Table \ref{table:word_change} lists the words most influenced by INLP projection (on average over all professions) after the debiasing procedure explained in Section \ref{sec:bios}.

Figure \ref{fig:rel-change}
presents the relative change of biased word for each profession, compared to a random sample.
\begin{center}
\begin{table}[H]
\centering
\begin{tabular}{|l|}
\hline
Most Changed Words \\ \hline
ms., mr., his, her, he, she, mrs., specializes, \\
english, practices, ',', him, spanish, \\
speaks, with, affiliated, and, medicine, ms, \\
state, \#, the, medical, michael, in, \\
residency, at, of, psychology, dr., ’s, \\
law, research, practice, about, where, \\
business, education, 5, -, is, first, \\
women, america, insurance, more, john, \\
university, location, ph.d., surgery, (, \\
mental, ), that, engineering, graduated, \\
language, bs, litigation, collection, \\
united, 1, graduate, humana, cpas, \\
cancer, npi, completed, 10, book, hospital, c, \\
out, family, or, when, oklahoma, certified, \\
ohio, number, training, for, like, a, \\
than, be, nursing, ], \_, can, writing, \\
patients, no, orthopaedic, attorney, \\
over, ny, mr, “,    \\ \hline
\end{tabular}

\caption{Top 100 words influenced by INLP projection (BOW representation, biographies dataset).}
\label{table:word_change}
\end{table}
\end{center}

\begin{figure}[H]
    \centering
    \includegraphics[width=1.0\columnwidth]{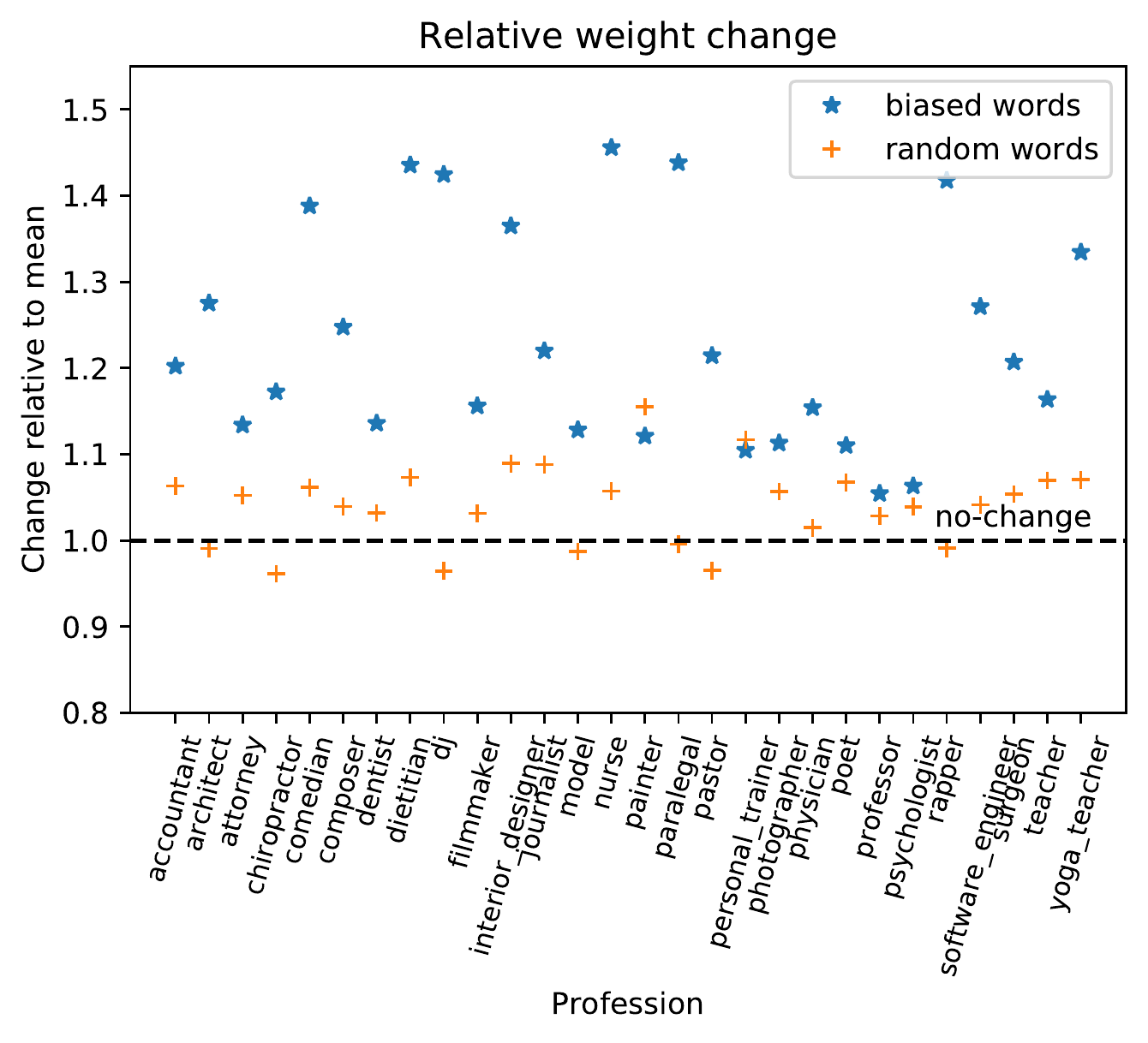}
    \caption{The relative change of biased vs. random words, per profession. }
    \label{fig:rel-change}
\end{figure}


\subsubsection{Bag-of-Word-Vectors Model}
\label{sec:appendix-bios-bag-of-vectors}

\begin{table*}[ht]
\centering
\resizebox{\textwidth}{!}{%
\begin{tabular}{cll}
\toprule
 Gender direction &                                      Male-biased &                                 Female-biased \\
\midrule
 0 &  his, he, His, himself &  herself, she, She, her \\
 1 &  himself, him, Gavin, His &  Hatha, midwifery, Midwifery, feminist \\
 2 &  Mark, Jon, Darren, Luke &  Actress, Zumba, Diana, woman \\
 3 &  Gordon, he, wind, charge &  hers, recipe, Challenge, cookbooks \\
 4 &  1935, 1955, namely, 1958 &  Roots, Issue, FHM, yoga \\
 5 &  M, Mickey, KS, Bethesda &  Vietnam, Subject, Elle, Ecuador \\
 6 &  Keys, correct, address, fuel &  leap, Embedded, textile, femininity \\
 7 &  Papers, Categories, wherein, Newark &  Botox, LASIK, periodontal, UnityPoint \\
 8 &  binding, closely, MT, command &  Aventura, brunette, HTML, Disclosure \\
 9 &  82, 92, 91, 86 &  ASP.NET, committer, Twilight, Seth \\
 10 &  t, Cisco, Philips, Sharp &  preschool, caregivers, homeowners, Preschool \\
 11 &  Toulouse, Aviv, scored, commended &  intersectional, Equality, equality, ASME \\
 12 &  addressing, segment, inequalities, segments &  Wire, loose, anything, Vincents \\
 13 &  comparison, Hart, 480, refereed &  Matthew, independence, couples, LGBTQ \\
 14 &  manufacturer, organizers, scope, specifications &  homeschooling, ligament, loyalty, graduating \\
\bottomrule
\end{tabular}
}
\caption{words closest to top 15 INLP gender directions (FastText representation, biographies dataset).}
\label{table:nearest-fasttext-gender-bios}
\end{table*}

In this section, we present an analysis for the influence of INLP projection on the FastText representation of individual words, under the bag-of-word-vectors model, for the biographies experiments (Section \ref{sec:bios-results}). We begin by ordering the vocabulary items by their cosine similarity to each of the top 15 gender directions found in INLP (i.e., their similarity to the weight vector of each classifier). For each gender direction $w_i$, we focus on the 20,000 most common vocabulary items, and calculate the closest words to $w_i$ (to get male-biased words) as well as the closest words to $-w_i$ (to get female-biased words). The result are presented in Table \ref{table:nearest-fasttext-gender-bios}.

The first gender direction seems to capture pronouns. Other gender directions capture socially biased terms, such as ``preschool" (direction 10), ``cookbooks" (direction 3) or other gender-related terms, such as ``LGBTQ" (direction 15) or ``femininity" (direction 6). Interestingly, those are mostly female-biased terms. As for the male-biased words, some directions capture surnames, such as ``Gordon" and ``Aviv". Other words which were found to be male-biased are less interpretable, such as words specifying years (direction 4), organizational terms such as ``Organizers", ``specifications" (direction 14), or the words ``Papers", ``Categories" (direction 7). It is not clear if those are the result of spurious correlations/noise, or whether they reflect actual subtle differences in the way the biographies of men and women are written.


\paragraph{Gender rowspace} The above analysis focuses on what information do individual gender directions convey. Next, we aim to demonstrate the influence of the final INLP projection on the representation of words. To this end, we rely on the \emph{rowspace} of the INLP matrix $P$. Recall that the rowspace is the orthogonal complement of the nullspace. As the INLP matrix $P$ projects to the intersection of nullspaces of the gender directions, the complement $P_R := I - P$ projects to the union of rowspaces of individual gender directions. This is a subspace which is spanned by all gender directions, and thus can be thought of as an empirical gender subspace within the representation space.

For a given word vector $w$, the ``gender norm" -- the norm of its projection on the rowspace, $||P_Rw||$ -- is a scalar quantity which can serve as a measure for the gender-bias of the word. We sort the vocabulary by the ratio between the gender norm and the original norm, $\frac{||P_Rw||}{||w||}$ and present the 200 most gendered words (Table \ref{table:top-gendered-rowspace}).

\begin{center}
\begin{table}[ht]
\resizebox{\columnwidth}{!}{%
\begin{tabular}{|l|}
\hline
Top words by component on the gender subspace \\ \hline
motherhood, SSHRC, \\ microfinance, preschool, genocide, IFP, \\ CSE, intersectional, student, \\ homeschooling, photoshoot, \\ intersectionality, 920, breastfeeding, STEM, \\ photojournalistic, haiku, kindergarten, \\ FreeOnes, UNESCO, menstrual, \\ turbulence, NTR, ASME, HFN, ECE, IEEE, \\ feminism, noir, Jadavpur, Motherhood, \\ reportage, Contra, TU, WebSphere, \\ counsellor, photovoltaic, J2EE, \\ contraception, university, PEN, \\ masculinities, parenting, EAP, \\ Politecnico, Feminism, trauma, \\ Universiti, counselling, curriculum, \\ Kanpur, women, edits, Pune, Nanjing, \\ ethnographic, Pinterest, surrealist, \\ taught, Hindustan, students, CNRS, \\ Bangalore, Mumbai, consortium, tooth, \\ Vitae, Kindergarten, nanoscale, \\ school, ACL, scholarships, cloud, \\ Goa, NIJC, Montessori, JSPS, \\ scholarship, Neha, DAAD, endometriosis, \\ carrier, UCI, activism, Ambedkar, \\ EECS, semiconductor, scholar, \\ microfluidic, bikini, Raising, teacher, \\ Feminist, vinyasa, NBER, ethnography, \\ Twilight, Sunil, Shankar, viral, \\ earthquake, semiconductors, \\ historiography, vampire, HMO, PSU, bioenergy, \\ historian, Ravi, Breastfeeding, Raman, \\ resettlement, Shweta, ICTs, UNDP, NVIDIA, \\ HIV, Counselling, HEC, KDD, \\ Hyderabad, contraceptive, macro, \\ Ghaziabad, sexuality, CAS, \\ documentary, mic, biography, postdoc, \\ transnationalism, AMD, CFD, B.Tech, physicist, \\ LGBT, parenthood, HKU, HIP, \\ internationalization, M.Tech, BDS, acne, theorist, \\ HPV, Meerut, ageing, smile, \\ Rajesh, psychoeducational, PUNE, \\ grief, AHA, Essays, discourses, \\ secrets, Swati, EPFL, coaching, IIE, \\ Manoj, BIDMC, infertility, \\ fashion, Chicana, Vaishali, \\ Graduation, sociologist, Gender, EA, MIT, \\ teach, gift, IETF, NPPA, counselor, \\ JPL, gender, menopause, LGBTQ, \\ Waseda, perceptions, praxis, \\ birthday, Jawaharlal, fertility, \\ gendered, coverage, stills, PIH, \\ Balaji, Tagged, baking, USM, \\ postpartum, Goenka, Pooja, forgiveness \\
\hline

\end{tabular}
}
\caption{Words by gender norm.}
\label{table:top-gendered-rowspace}
\end{table}
\end{center}

 As before, we see a combination of inherently-gendered words (``motherhood", ``women", ``gender", ``masculinities"), socially-biased terms (``teacher", ``raising", ``semiconductors", ``B.Tech", ``IEEE", ``STEM", ``fashion") and other words whose connection to gender is less interpretable, and potentially represent spurious correlations (``trauma", ``Vitae", ``smile", ``920", ``forgiveness"). 

\end{document}